\documentclass{article} 
\usepackage{iclr2020_conference,times}
\usepackage{graphicx}
\usepackage{hhline}
\usepackage{wrapfig}

\usepackage{amsmath,amsfonts,bm}

\def\figref#1{figure~\ref{#1}}

\def\secref#1{section~\ref{#1}}
\def\Secref#1{Section~\ref{#1}}

\def\1{\bm{1}}

\DeclareMathAlphabet{\mathsfit}{\encodingdefault}{\sfdefault}{m}{sl}
\SetMathAlphabet{\mathsfit}{bold}{\encodingdefault}{\sfdefault}{bx}{n}

\newcommand{\R}{\mathbb{R}}

\DeclareMathOperator*{\argmin}{arg\,min}

\definecolor{ourspecialtextcolor}{rgb}{0.528, 0.471, 0.701} 
\usepackage{hyperref}
\hypersetup{
	colorlinks,
	linkcolor=teal,
	urlcolor=teal,
	citecolor=ourspecialtextcolor}

\usepackage{url}
\usepackage{caption,subcaption}
\usepackage{makecell}

\usepackage{todonotes}

\usepackage{enumitem}
\setlist[enumerate]{label={\rm(\roman*)},leftmargin=6ex}
\usepackage{multirow}

\usepackage{algorithm}
\usepackage[noend]{algpseudocode}
\DeclareCaptionSubType*{algorithm}

\algrenewcommand{\algorithmiccomment}[1]{\bgroup\hfill//~#1\egroup}
\algrenewcommand{\Return}{\State\textbf{return}\ }
\algnewcommand{\Save}{\State\textbf{save}\ }
\algnewcommand{\Load}{\State\textbf{load}\ }

\renewcommand{\figref}[1]{Fig.~\ref{#1}}
\newcommand{\tbref}[1]{Tab.~\ref{#1}}
\newlength{\picHeight}

\newcommand{\WW}{\mathcal W}

\renewcommand{\d}{\mathrm{d}}
\newcommand{\cc}{\mathbf{c}}
\newcommand{\il}{\frac{1}{\lambda}}
\newcommand{\dccdw}{\frac{\d\cc}{\d w}}

\newcommand{\pdiff}[2]{\frac{\d #1}{\d #2}}

\newcommand{\std}[1]{ \pm #1}

\DeclareMathOperator{\dist}{dist}
\newcommand\Wd[1][\lambda]{W^{#1}_{\text{dif}}}
\newcommand\We[1][\lambda]{W^{#1}_{\text{eq}}}

\newcommand{\wy}{\bigl(w,y(w)\bigr)}
\newcommand{\wyl}{\bigl(w,y_\lambda(w)\bigr)}

\newcommand{\yw}{\bigl(y(w)\bigr)}
\newcommand{\ylw}{\bigl(y_\lambda(w)\bigr)}

\usepackage{amsthm}

\theoremstyle{plain}
\newtheorem{theorem}{Theorem}
\newtheorem{prop}{Proposition}
\newtheorem{example}{Example}
\newtheorem{remark}{Remark}

\newtheorem{observation}{Observation}
\newtheoremstyle{MY}{}{}{\upshape}{}{\bfseries}{.}{5pt plus 4pt minus 3pt}{\thmname{#1}\thmnumber{ A#2#3}}
\theoremstyle{MY}
\newtheorem{property}{Property}

\expandafter\let\expandafter\oldproof\csname\string\proof\endcsname
\let\oldendproof\endproof
\renewenvironment{proof}[1][\proofname]{%
  \oldproof[{{\bf #1.}}]%
}{\oldendproof}

\usepackage{xspace}
\makeatletter
\DeclareRobustCommand\onedot{\futurelet\@let@token\@onedot}
\def\@onedot{\ifx\@let@token.\else.\null\fi\xspace}

\def\eg{{e.g}\onedot}  
\def\ie{{i.e}\onedot}

\makeatother
\newcommand\email[1]{{\sl #1}}

\title{Differentiation of Blackbox Combinatorial Solvers}

\author{%
	Marin Vlastelica$^1$
	\thanks{These authors contributed equally.}\ \ ,
	Anselm Paulus$^{1*}$,
	V\'\i t Musil$^2$,
	Georg Martius$^1$,
	Michal Rol\'inek$^1$
	\\[3pt]
	$^1$ Max-Planck-Institute for Intelligent Systems,
	T\"ubingen, Germany\\
	$^2$ Universit\`a degli Studi di Firenze, Italy
	\\[3pt]
	\email{\{marin.vlastelica, anselm.paulus, georg.martius, michal.rolinek\}@tuebingen.mpg.de}\\
	\email{vit.musil@unifi.it}\\
	\And
}

\iclrfinalcopy 

\begin{document}

%
\maketitle

\begin{abstract}
Achieving fusion of deep learning with combinatorial algorithms promises
transformative changes to artificial intelligence.  One possible approach is to
introduce combinatorial building blocks into neural networks. Such end-to-end
architectures have the potential to tackle combinatorial problems on raw input
data such as ensuring global consistency in multi-object tracking or route
planning on maps in robotics.  In this work, we present a method that
implements an efficient backward pass through blackbox implementations of
combinatorial solvers with linear objective functions.  We provide both
theoretical and experimental backing. In particular, we incorporate the Gurobi
MIP solver, Blossom V algorithm, and Dijkstra's algorithm into architectures
that extract suitable features from raw inputs for the traveling salesman
problem, the min-cost perfect matching problem and the shortest path problem.
The code is available at
\begin{center}
	\url{https://github.com/martius-lab/blackbox-backprop}.
\end{center}
\end{abstract}

\section{Introduction}

The toolbox of popular methods in computer science currently sees a split into
two major components. On the one hand, there are classical algorithmic
techniques from discrete optimization -- graph algorithms, SAT-solvers, integer
programming solvers -- often with heavily optimized implementations and
theoretical guarantees on runtime and performance. On the other hand, there is
the realm of deep learning allowing data-driven feature extraction as well as
the flexible design of end-to-end architectures. The fusion of deep learning
with combinatorial optimization is desirable both for foundational reasons --
extending the reach of deep learning to data with large combinatorial
complexity -- and in practical applications. These often occur for example in
computer vision problems that require solving a combinatorial sub-task on top
of features extracted from raw input such as establishing global consistency in
multi-object tracking from a sequence of frames.

The fundamental problem with constructing hybrid architectures is
differentiability of the combinatorial components. State-of-the-art approaches
pursue the following paradigm: introduce suitable approximations or
modifications of the objective function or of a baseline algorithm that
eventually yield a differentiable computation. The resulting algorithms are
often sub-optimal in terms of runtime, performance and optimality guarantees
when compared to their \emph{unmodified} counterparts. While the sources of
sub-optimality vary from example to example, there is a common theme: any
differentiable algorithm in particular outputs continuous values and as such it
solves a \emph{relaxation} of the original problem. It is well-known in
combinatorial optimization theory that even strong and practical convex
relaxations induce lower bounds on the approximation ratio for large classes of
problems \citep{Raghavendra:2008:OAI:1374376.1374414,
Thapper:2017:LSR:3329995.3330022} which makes them inherently sub-optimal.
This inability to incorporate the best implementations of the best algorithms is unsatisfactory.

\begin{figure}[h]
	\centering
	\includegraphics[width=0.95\linewidth]{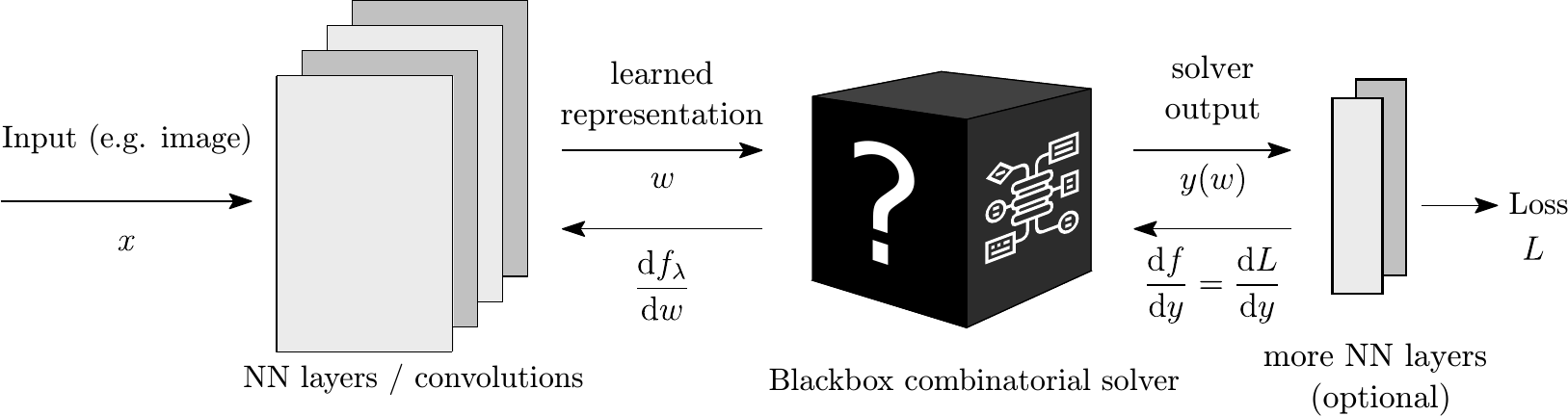}
	\caption{Architecture design enabled by Theorem \ref{T:f-lambda}. Blackbox
	combinatorial solver embedded into a neural network.}
	\label{fig:architecture}
\end{figure}

In this paper, we propose a method that, at the cost of one hyperparameter,
implements a backward pass for a {\bf blackbox implementation} of a
combinatorial algorithm or a solver that optimizes a linear objective function.
This effectively turns the algorithm or solver into a composable building block
of neural network architectures, as illustrated in \figref{fig:architecture}.
Suitable problems with linear objective include classical problems such as {\sc
shortest-path}, {\sc traveling-salesman} (TSP), {\sc
min-cost-perfect-matching}, various cut problems as well as entire frameworks
such as integer programs (IP), Markov random fields (MRF) and conditional
random fields (CRF).

The main technical challenge boils down to providing an informative gradient of
a piecewise constant function. To that end, we are able to heavily leverage the
minimization structure of the underlying combinatorial problem and efficiently
compute a gradient of a continuous interpolation. While the roots of the method
lie in loss-augmented inference, the employed mathematical technique for
continuous interpolation is novel. The computational cost of the introduced
{\bf backward pass matches the cost of the forward pass}. In particular, it
also amounts to one call to the solver.

In experiments, we train architectures that contain {\bf unmodified
implementations} of the following efficient combinatorial algorithms:
general-purpose mixed-integer programming solver Gurobi \citep{gurobi},
state-of-the-art C implementation of {\sc min-cost-perfect-matching} algorithm
-- Blossom V \citep{Kolmogorov2009} and Dijkstra's algorithm \citep{Dijkstra}
for {\sc shortest-path}.  We demonstrate that the resulting architectures train
without sophisticated tweaks and are able to solve tasks that are beyond the
capabilities of conventional neural networks.

\section{Related Work} \label{sec:related-work}

Multiple lines of work lie at the intersection of combinatorial algorithms and
deep learning. We primarily distinguish them by their motivation.

\paragraph{Motivated by applied problems.} Even though computer vision has seen
a substantial shift from combinatorial methods to deep learning, some problems
still have a strong combinatorial aspect and require hybrid approaches.
Examples include multi-object tracking \citep{Schulter2017DeepNF}, semantic
segmentation \citep{7913730}, multi-person pose estimation
\citep{deepcut16cvpr, Jie:graph:decomposition:2018}, stereo matching
\citep{Knbelreiter2016EndtoEndTO} and person re-identification
\citep{Ye_2017_ICCV}. The combinatorial algorithms in question are typically
Markov random fields (MRF) \citep{Chen:2015:LDS:3045118.3045308}, conditional
random fields (CRF) \citep{Marin_2019_CVPR}, graph matching
\citep{Ye_2017_ICCV} or integer programming \citep{Schulter2017DeepNF}. In
recent years, a plethora of hybrid end-to-end architectures have been proposed.
The techniques used for constructing the backward pass range from employing
various relaxations and approximations of the combinatorial problem
\citep{Chen:2015:LDS:3045118.3045308, Zheng:2015:CRF:2919332.2919659} over
differentiating a fixed number of iterations of an iterative solver
\citep{Paschalidou2018CVPR, Tompson:2014:JTC:2968826.2969027,
Liu:2015:SIS:2919332.2920082} all the way to relying on the structured SVM
framework \citep{Tsochantaridis:2005:LMM:1046920.1088722,
Chen:2015:LDS:3045118.3045308}.

\paragraph{Motivated by ``bridging the gap''.} Building links between
combinatorics and deep learning can also be viewed as a foundational problem;
for example, \citep{47094} advocate that ``combinatorial generalization must be
a top priority for AI''. One such line of work focuses on designing
architectures with algorithmic structural prior -- for example by mimicking the
layout of a Turing machine \citep{memory-networks-e2e, pointer-networks,
neural-turing-machine, graves2016hybrid} or by promoting behaviour that
resembles message-passing algorithms as it is the case in Graph Neural Networks
and related architectures \citep{Scarselli, li2016gated, 47094}. Another
approach is to provide neural network building blocks that are specialized to
solve some types of combinatorial problems such as satisfiability (SAT)
instances \citep{wang2019satnet},  mixed integer programs \citep{miplayer},
sparse inference \citep{pmlr-v80-niculae18a}, or submodular maximization
\citep{tschiatschek18differentiable}. A related mindset of learning inputs to
an optimization problem gave rise to the ``predict-and-optimize'' framework and
its variants \citep{predict-optimize, predict-optimize-ranking,
predict-and-optimize-comb}. Some works have directly addressed the question of
learning combinatorial optimization algorithms such as the {\sc
traveling-salesman-problem} in \citep{neural-comb-with-rl} or its vehicle
routing variants \citep{NIPS2018_8190}. A recent approach also learns
combinatorial algorithms via a clustering proxy \citep{NIPS2019_8715}.

There are also efforts to bridge the gap in the opposite direction; to use deep
learning methods to improve state-of-the-art combinatorial solvers, typically
by learning (otherwise hand-crafted) heuristics. Some works have again targeted
the {\sc traveling-salesman-problem} \citep{kool2018attention,
tsp-policy-gradient, neural-comb-with-rl} as well as other NP-Hard problems
\citep{np-hard-learned-solvers}. Also, more general solvers received some
attention; this includes SAT-solvers \citep{guiding-sat-solver,
selsam2018learning}, integer programming solvers (often with learning
branch-and-bound rules) \citep{branch-and-bound, learn-branch,
DBLP:journals/corr/abs-1906-01629} and SMT-solvers (satisfiability modulo
theories)\citep{smt-solver}.

\section{Method}

Let us first formalize the notion of a combinatorial solver. We expect the
solver to receive continuous input $w \in W\subseteq \R^N$ (\eg edge weights of
a fixed graph) and return discrete output $y$ from some finite set $Y$ (\eg all
traveling salesman tours on a fixed graph) that minimizes some cost $\cc(w,y)$
(\eg length of the tour). More precisely, the solver maps
\begin{equation} \label{E:solver}
       w \mapsto y(w)
              \quad \text{such that} \quad
       y(w) = \argmin_{y\in Y} \cc(w,y).
\end{equation}

We will restrict ourselves to objective functions $\cc(w,y)$ that are
\textbf{linear} , namely $\cc(w,y)$ may be represented as
\begin{equation} \label{E:c-linear}
       \cc(w,y) = w\cdot \phi(y)
       \quad\text{for $w\in W$ and $y\in Y$}
\end{equation}
in which $\phi\colon Y\to \R^N$ is an injective representation of $y\in Y$ in
$\R^N$.  For brevity, we omit the mapping $\phi$ and instead treat elements of
$Y$ as discrete points in $\R^N$.

Note that such definition of a solver is still very general as there are
\textbf{no assumptions on the set of constraints or on the structure of the
output space $Y$.}

\begin{example}[Encoding shortest-path problem] \label{ex:shortest-path}
If $G = (V, E)$ is a given graph with vertices $s, t \in V$, the combinatorial
solver for the $(s,t)$-{\sc shortest-path} would take edge weights $w \in W =
\R^{|E|}$ as input and produce the shortest path $y(w)$ represented as $\phi(y)
\subseteq \{0, 1\}^{|E|}$ an indicator vector of the selected edges. The cost
function is then indeed the inner product $\cc(w,y) = w \cdot \phi(y)$.
\end{example}

The task to solve during back-propagation is the following. We receive the
gradient $\d L/\d y$ of the global loss $L$ with respect to solver output $y$
at a given point $\hat y=y(\hat w)$. We are expected to return $\d L/\d w$, the
gradient of the loss with respect to solver input $w$ at a point $\hat w$.

Since $Y$ is finite, there are only finitely many values of $y(w)$. In other
words, this function of $w$ is \textbf{piecewise constant} and the gradient is
identically zero or does not exist (at points of jumps). This should not come
as a surprise; if one does a small perturbation to edge weights of a graph, one
\textit{usually} does not change the optimal TSP tour and \textit{on rare
occasions} alters it drastically.  This has an important consequence:

\begin{quote}
The fundamental problem with differentiating through combinatorial solvers is
not the lack of differentiability; the gradient exists \textit{almost
everywhere}. However, this gradient is a constant zero and as such is unhelpful
for optimization.
\end{quote}

Accordingly, we will \textit{not rely} on standard techniques for gradient
estimation (see \citep{mohamed2019monte} for a comprehensive survey).

First, we simplify the situation by considering the linearization $f$ of $L$ at
the point $\hat{y}$. Then for
\begin{equation*}
	f(y) = L(\hat{y}) + \frac{\d L}{\d y}(\hat y) \cdot (y-\hat{y})
	\quad\text{we have}\quad
	\pdiff{f\yw}{w} = \pdiff{L}{w}
\end{equation*}
and therefore it suffices to focus on differentiating the piecewise constant
function $f\yw.$

If the piecewise constant function at hand was arbitrary, we would be forced to
use zero-order gradient estimation techniques such as computing finite
differences.  These require prohibitively many function evaluations
particularly for high-dimensional problems.

However, the function $f\yw$ is a result of a minimization process and it is
known that for smooth spaces $Y$ there are techniques for such
``differentiation through argmin'' \citep{schmidt2014shrinkage,
samuel2009learning,foo2008efficient, domke2012generic, amos2017input,
br2017optnet}. It turns out to be possible to build -- with different
mathematical tools -- a viable discrete analogy. In particular, we can
efficiently {\bf construct a function} $f_\lambda(w)$, {\bf a continuous
interpolation} of $f\yw$, whose gradient we return (see \figref{fig:f-lambda}).
The hyper-parameter $\lambda > 0$ controls the trade-off between
``informativeness of the gradient'' and ``faithfulness to the original
function''.

Before diving into the formalization, we present the final algorithm as listed
in Algo.~\ref{algo:main}. It is simple to implement and the backward pass
indeed only runs the solver once on modified input. Providing the
justification, however, is not straightforward, and it is the subject of the
rest of the section.

\algrenewcommand\algorithmicindent{1em}
\algrenewcommand{\algorithmiccomment}[1]{\bgroup\hskip2em\textcolor{ourspecialtextcolor}{//~\textsl{#1}}\egroup}
\begin{algorithm}
\begin{subalgorithm}[t]{.5\textwidth}
\begin{algorithmic}
       \Function{ForwardPass}{$\hat w$}
       \State
       $\hat y :=$ \textbf{Solver}($\hat w$)  \Comment{$\hat y = y(\hat w)$}
       \Save $\hat w$ and $\hat y$ for backward pass
       \Return $\hat y$
       \EndFunction
\end{algorithmic}
\end{subalgorithm}%
\begin{subalgorithm}[t]{.5\textwidth}
\begin{algorithmic}
       \Function{BackwardPass}{$\tfrac{\d L}{\d y}(\hat y)$, $\lambda$}
       \Load $\hat w$ and $\hat y$ from forward pass
       \State $w' := \hat w + \lambda \cdot \tfrac{\d L}{\d y}(\hat y)$
       \\ \Comment{Calculate perturbed weights}
       \State $y_\lambda :=$ \textbf{Solver}($w'$)
       \Return $\nabla_w f_\lambda(\hat w) := -\il\bigl[ \hat y - y_\lambda\bigr]$
       \\ \Comment{Gradient of continuous interpolation}
       \EndFunction
\end{algorithmic}
\end{subalgorithm}
\caption{Forward and Backward Pass}\label{algo:main}%
\end{algorithm}

\subsection{Construction and Properties of $f_\lambda$}

\begin{figure}
  \centering
  \begin{subfigure}[b]{0.45\linewidth}
    \includegraphics[width=\linewidth]{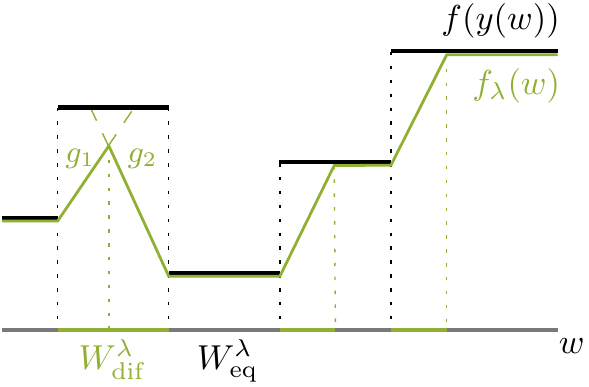}
    \caption{}
    \label{fig:d-lambda-small}
  \end{subfigure}
  \hfil
  \begin{subfigure}[b]{0.45\linewidth}
    \includegraphics[width=\linewidth]{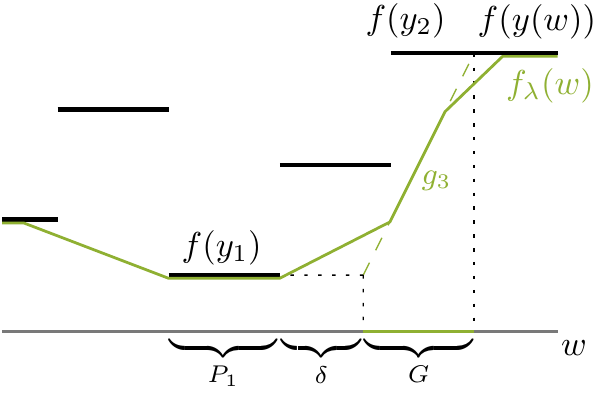}
    \caption{}
    \label{fig:d-lambda-big}
  \end{subfigure}
  \caption{Continuous interpolation of a piecewise constant function.
    (\subref{fig:d-lambda-small}) $f_\lambda$ for a small value of $\lambda$; the set $\We$ is still
    substantial and only two interpolators $g_1$ and $g_2$ are incomplete.
    Also, all interpolators are $0$-interpolators.
    (\subref{fig:d-lambda-big}) $f_\lambda$ for a high value of $\lambda$; most interpolators are
    incomplete and we also encounter a $\delta$-interpolator $g_3$ (between
    $y_1$ and $y_2$) which attains the value $f(y_1)$ $\delta$-away from the set
    $P_1$. 
    Despite losing some local structure for high
    $\lambda$, the gradient of $f_\lambda$ is still informative.
  }
  \label{fig:f-lambda}
\end{figure}

Before we give the exact definition of the function $f_\lambda$, we formulate
several requirements on it. This will help us understand why $f_\lambda(w)$ is
a reasonable replacement for $f\yw$ and, most importantly, why its gradient
captures changes in the values of $f$.

\begin{property}
For each $\lambda>0$, $f_\lambda$ is continuous and piecewise affine.
\end{property}

The second property describes the trade-off induced by changing the value of
$\lambda$. For $\lambda>0$, we define sets $\We$ and $\Wd$ as the sets where
$f\yw$ and $f_\lambda(w)$ coincide and where they differ, \ie
\begin{equation*}
       \We = \left\{w\in W: f_\lambda(w)=f\yw \right\}
       \quad\text{and}\quad
       \Wd = W\setminus\We.
\end{equation*}

\begin{property}
The sets $\Wd$ are monotone in $\lambda$ and they vanish as $\lambda\to 0^+$,
\ie
\begin{equation*}
       \Wd[\lambda_1] \subseteq \Wd[\lambda_2]
              \quad\text{for $0<\lambda_1\le \lambda_2$}
       \quad\text{and}\quad
       \Wd\to\emptyset
              \quad\text{as $\lambda\to 0^+$}.
\end{equation*}
\end{property}

In other words, Property A2 tells us that $\lambda$ controls the size of the
set where $f_\lambda$ deviates from $f$ and where $f_\lambda$ has meaningful
gradient. This behaviour of $f_\lambda$ can be seen in \figref{fig:f-lambda}.

In the third and final property, we want to capture the interpolation behavior
of $f_\lambda$. For that purpose, we define a \emph{$\delta$-interpolator} of
$f$.  We say that $g$, defined on a set $G\subset W$, is a
$\delta$-interpolator of $f$ between $y_1$ and $y_2\in Y$, if
\begin{itemize}
\item $g$ is non-constant affine function;
\item the image $g(G)$ is an interval with endpoints $f(y_1)$ and $f(y_2)$;
\item $g$ attains the boundary values $f(y_1)$ and $f(y_2)$ at most
$\delta$-far away from where $f(y(w))$ does. In particular, there is a point
$w_k\in G$ for which $g(w_k)=f(y_k)$ and $\dist(w_k,P_k)\le\delta$, where $P_k
= \{w\in W:y(w)=y_k\}$, for $k=1,2$.
\end{itemize}

In the special case of a 0-interpolator $g$, the graph of $g$ connects (in a
topological sense) two components of the graph of $f\yw$. In the general case,
$\delta$ measures \emph{displacement} of the interpolator (see also
\figref{fig:f-lambda} for some examples). This displacement on the one hand
loosens the connection to $f\yw$ but on the other hand allows for less local
interpolation which might be desirable.

\begin{property}
The function $f_\lambda$ consists of finitely many (possibly incomplete)
$\delta$-interpolators of $f$ on~$\Wd$ where $\delta\le C\lambda$ for some
fixed $C$. Equivalently, the \emph{displacement} is linearly controlled by
$\lambda$.
\end{property}

Intuitively, the consequence of Property A3 is that $f_\lambda$ has reasonable gradients everywhere since it consists of elementary affine interpolators.

For defining the function $f_\lambda$, we need a solution of a perturbed
optimization problem
\begin{equation} \label{E:yl-def}
	y_\lambda(w) = \argmin_{y\in Y} \{\cc(w,y)+\lambda f(y)\}.
\end{equation}

\begin{theorem} \label{T:f-lambda}
Let $\lambda>0$. The function $f_\lambda$ defined by
\begin{equation} \label{E:fl-def}
	f_\lambda(w) = f\ylw - \il\Bigl[ \cc\wy - \cc\wyl\Bigr]
\end{equation}
satisfies Properties A1, A2, A3.
\end{theorem}

Let us remark that already the continuity of $f_\lambda$ is not apparent from
its definition as the first term  $f\ylw$ is still a piecewise constant
function. Proof of this result, along with geometrical description of
$f_\lambda$, can be found in \secref{sec:app:proofs}.  \figref{fig:f-lambda-2d}
visualizes $f_{\lambda}$ for different values if $\lambda$.
\begin{figure}
  \centering
  \includegraphics[width=0.8\linewidth]{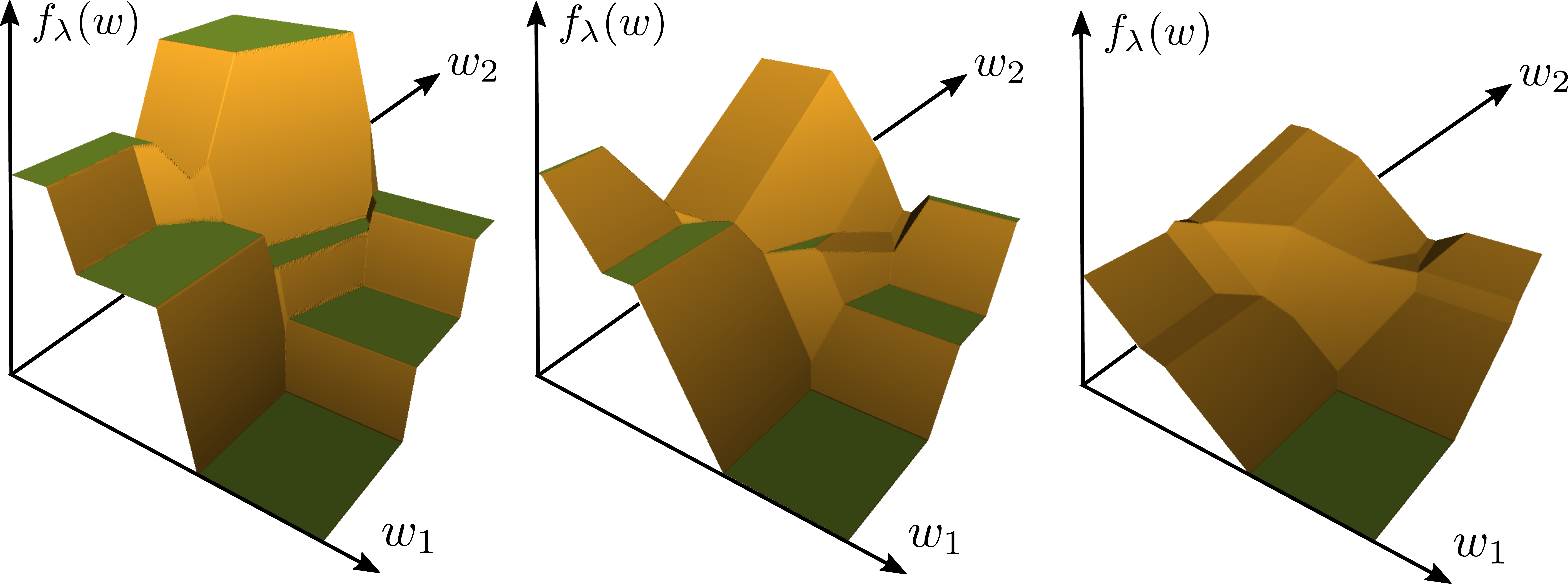}
	\caption{Example $f_\lambda$ for $w \in \R^2$ and $\lambda = 3, 10, 20$ (left
	to right). As $\lambda$ changes, the interpolation $f_\lambda$ is less
	faithful to the piecewise constant $f(y(w))$ but provides reasonable gradient
	on a larger set.
  }
  \label{fig:f-lambda-2d}
\end{figure}

Now, since $f_\lambda$ is ensured to be differentiable, we have
\begin{equation} \label{E:update-rule}
	\nabla f_\lambda(w)
		= -\il\Bigl[ \dccdw\wy - \dccdw\wyl\Bigr]
		= -\il\bigl[ y(w) - y_\lambda(w)\bigr].
\end{equation}
The second equality then holds due to \eqref{E:c-linear}.  We then return
$\nabla f_\lambda$ as a loss gradient.

\begin{remark}
The roots of the method we propose lie in loss-augmented inference. In fact,
the update rule from \eqref{E:update-rule} (but not the function $f_\lambda$ or
any of its properties) was already proposed in a different context in
\citep{NIPS2010_4069, pmlr-v48-songb16} and was later used in
\citep{opt-through-argmax, Mohapatra2016EfficientOF}. The main difference to our
work is that only the case of $\lambda \to 0^+$ is recommended and studied,
which in our situation computes the correct but uninformative zero gradient.
Our analysis implies that {\bf larger values of $\lambda$ are not only sound
but even preferable}. This will be seen in experiments where we use values
$\lambda \approx 10-20$.
\end{remark}

\subsection{Efficient Computation of $f_\lambda$}

Computing $y_\lambda$ in \eqref{E:yl-def} is the only potentially expensive
part of evaluating \eqref{E:update-rule}. However, the linear interplay of the
cost function and the gradient trivially gives a resolution.

\begin{prop}\label{prop:linear}
Let $\hat w \in W$ be fixed. If we set $w' = \hat w + \lambda\tfrac{\d L}{\d
y}(\hat y)$, we can compute $y_\lambda$ as
\begin{equation*}
	y_\lambda(\hat w) = \argmin_{y\in Y} \cc(w',y).
\end{equation*}
\end{prop}

In other words, $y_\lambda$ is the output of calling the solver on input $w'$.

\section{Experiments}

In this section, we {\bf experimentally validate a proof of concept}: that
architectures containing exact blackbox solvers (with backward pass provided by
Algo.~\ref{algo:main}) can be trained by standard methods.

\begin{table*}[h]
	\caption{Experiments Overview.}
	\label{tab:experiments}
	\centering
	\begin{tabular}{@{\ }cccc@{\ }}
	Graph Problem & Solver & Solver instance size & Input format \\
	\hline
	Shortest path & Dijkstra & up to $900$ vertices & (image) up to $240 \times 240$ \\
	Min Cost PM & Blossom V & up to $1104$ edges & (image) up to $528\times528$ \\
	Traveling Salesman & Gurobi & up to $780$ edges & up to $40$ images ($20\times40$) \\
	\end{tabular}
\end{table*}

To that end, we solve three synthetic tasks as listed in
\tbref{tab:experiments}. These tasks are designed to mimic practical examples
from \Secref{sec:related-work} and solving them anticipates a two-stage
process:  1)~extract suitable features from raw input, 2)~solve a combinatorial
problem over the features. The dimensionalities of input and of intermediate
representations also aim to mirror practical problems and are chosen to be
prohibitively large for zero-order gradient estimation methods. Guidelines of
setting the hyperparameter $\lambda$ are given in
\secref{sec:app:guidlinelambda}.

We include the performance of ResNet18 \citep{He2015DeepRL} as a sanity check to
demonstrate that the constructed datasets are too complex for standard
architectures.

\begin{remark}
The included solvers have very efficient implementations and do not severely
impact runtime. All models train in under two hours on a single machine with 1
GPU and no more than 24 utilized CPU cores. Only for the large TSP problems the
solver's runtime dominates.
\end{remark}

\subsection{Warcraft Shortest Path}

\paragraph{Problem input and output.} The training dataset for problem {\sc
SP}$(k)$ consists of 10000 examples of randomly generated images of terrain
maps from the Warcraft II tileset \citep{warcraft}. The maps have an underlying
grid of dimension $k \times k$ where each vertex represents a terrain with a
fixed cost that is unknown to the network.  The shortest (minimum cost) path
between top left and bottom right vertices is encoded as an indicator matrix
and serves as a label (see also  \figref{fig:sp-results}).  We consider
datasets {\sc SP}$(k)$ for $k \in \{12, 18, 24, 30\}$. More experimental
details are provided in \secref{se:app:exp}.

\setlength{\fboxsep}{0pt}
\setlength{\picHeight}{0.25\linewidth}
\begin{figure}[htb]
    \centering
    \begin{subfigure}[c]{0.64\linewidth}
			\centering
			\parbox[b][\picHeight][c]{1em}{\rotatebox{90}{Input}}
			\includegraphics[height=\picHeight]{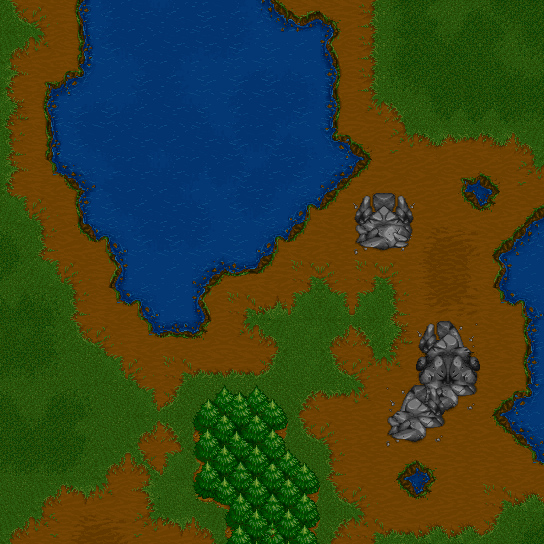}
			\vbox to\picHeight{\vfil\hbox{\LARGE$\to$}\vfil}
			\parbox[b][\picHeight][c]{\picHeight}{
			\centering
			Label
			$$
				\begin{pmatrix}
					1&0&\cdots&0\\
					1&0&\cdots&0\\
					\vdots & \vdots & \ddots &\vdots\\
					0&0&\cdots&1
				\end{pmatrix}
			$$
			$k\times k$ indicator matrix of shortest path}
			\caption{}
			\label{fig:sp-results:dataset}
    \end{subfigure}
		\hspace{.04\linewidth}
    \begin{subfigure}[c]{0.3\linewidth}
			\centering
			\includegraphics[height=\picHeight]{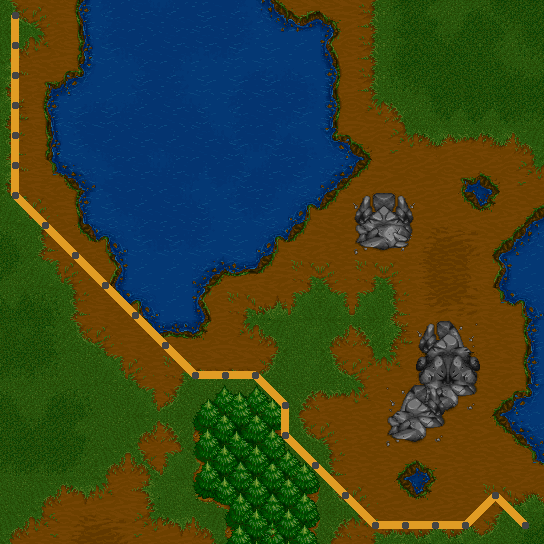}
			\caption{}
			\label{fig:sp-results:withpath}
    \end{subfigure}
		\caption{The {\sc SP}$(k)$ dataset. (\subref{fig:sp-results:dataset}) Each
		input is a $k \times k$ grid of tiles corresponding to a Warcraft II
		terrain map, the respective label is a the matrix indicating the shortest
		path from top left to bottom right. (\subref{fig:sp-results:withpath}) is a
		different map with correctly predicted shortest path.}
    \label{fig:sp-results}
\end{figure}

\paragraph{Architecture.} An image of the terrain map is presented to a
convolutional neural network which outputs a $k \times k $ grid of vertex
costs.  These costs are then the input to the Dijkstra algorithm to compute the
predicted shortest path for the respective map.  The loss used for computing
the gradient update is the Hamming distance between the true shortest path and
the predicted shortest path.

\begin{wraptable}{r}{.6\textwidth}
	\setlength{\belowcaptionskip}{-10pt plus 10pt minus 10pt}
	\centering
	\caption{{\bf Results for Warcraft shortest path.} Reported is the accuracy,
	\ie percentage of paths with the optimal costs. Standard deviations are over
	five restarts.}
	\label{tab:sp-results}
	\begin{tabular}{@{\ }ccccc@{\ }}
	& \multicolumn{2}{c}{Embedding Dijkstra}&\multicolumn{2}{c}{ResNet18}\\
	$k$ & Train \% & Test \% & Train \%& Test \%\\
	\hline
	12 & $99.7\std{0.0}$ & $96.0\std{0.3}$ & $100.0\std{0.0}$ & $23.0\std{0.3}$\\
	18 & $98.9\std{0.2}$ & $94.4\std{0.2}$ & $99.9\std{0.0}$  & $0.7\std{0.3}$ \\
	24 & $97.8\std{0.2}$ & $94.4\std{0.6}$ & $100.0\std{0.0}$ & $0.0\std{0.0}$ \\
	30 & $97.4\std{0.1}$ & $94.0\std{0.3}$ & $95.6\std{0.5}$  & $0.0\std{0.0}$ \\
	\end{tabular}
\end{wraptable}

\paragraph{Results.} Our method learns to predict the shortest paths with high
accuracy and generalization capability, whereas the ResNet18 baseline
unsurprisingly fails to generalize already for small grid sizes of $k=12$.
Since the shortest paths in the maps are often nonunique (i.e. there are
multiple shortest paths with the same cost), we report the percentage of
shortest path predictions that have optimal cost. The results are summarized in
\tbref{tab:sp-results}.

\subsection{Globe Traveling Salesman Problem}

\paragraph{Problem input and output.} The training dataset for problem {\sc
TSP}$(k)$ consists of 10000 examples where the input for each example is a
$k$-element subset of fixed 100 country flags and the label is the shortest
traveling salesman tour through the capitals of the corresponding countries.
The optimal tour is represented by its adjacency matrix (see also
\figref{fig:tsp}). We consider datasets {\sc TSP}$(k)$ for $k \in \{5, 10, 20,
40\}$.

\setlength{\picHeight}{0.24\linewidth}
\begin{figure}[htb]
	\begin{subfigure}[c]{0.6\linewidth}
		\centering
		\parbox[b][\picHeight][c]{2em}{%
			\rotatebox{90}{\parbox{\picHeight}{\centering Input\\$k$ flags}}}
		\parbox[b][\picHeight][c]{1em}{%
			$\left\{\vbox to 0.5\picHeight{} \right.$}
		\includegraphics[height=\picHeight]{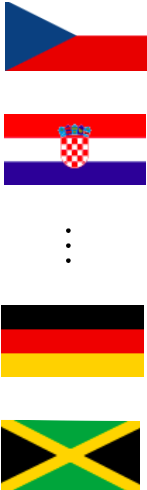}
		\parbox[b][\picHeight][c]{0.5\picHeight}{
			\centering
			\includegraphics[height=0.4\picHeight]{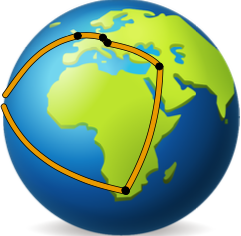}
			{\LARGE$\longrightarrow$}}
		\parbox[b][\picHeight][c]{\picHeight}{
		\centering
		Label
		$$
			\begin{pmatrix} 
				0 & \dots & 1 \\
				\vdots & \ddots & \vdots \\
				1 & \dots & 0 
			\end{pmatrix}
		$$
		$k \times k$ adjacency matrix with optimal TSP tour}
		\caption{}
		\label{fig:tsp:dataset}
	\end{subfigure}
	\hspace{.09\linewidth}
	\begin{subfigure}[c]{0.3\linewidth}
		\centering
		\includegraphics[height=\picHeight]{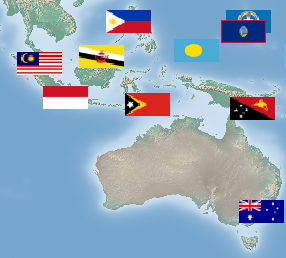}
		\caption{}
		\label{fig:tsp:oceania}
	\end{subfigure}
	\caption{The {\sc TSP}($k$) problem. (\subref{fig:tsp:dataset}) illustrates
	the dataset. Each input is a sequence of $k$ flags and the corresponding
	label is the adjacency matrix of the optimal TSP tour around the
	corresponding capitals. (\subref{fig:tsp:oceania}) displays the learned
	locations of 10 country capitals in southeast Asia and Australia, accurately
	recovering their true position.}
  \label{fig:tsp}
\end{figure}

\paragraph{Architecture.} Each of the $k$ flags is presented to a convolutional
network that produces $k$ three-dimensional vectors.  These vectors are
projected onto the unit sphere in $\R^3$; a representation of the globe.  The
TSP solver receives a matrix of pairwise distances of the $k$ computed
locations.  The loss of the network is the Hamming distance between the true
and the predicted TSP adjacency matrix. The architecture is expected to learn
the correct representations of the flags (\ie locations of the respective
countries' capitals on Earth, up to rotations of the sphere). The employed
Gurobi solver optimizes a mixed-integer programming formulation of TSP using
the cutting plane method \citep{Marchand:2002:CPI:772382.772395} for lazy
sub-tour elimination.

\begin{wraptable}{r}{.6\textwidth}
	\setlength{\belowcaptionskip}{-10pt plus 10pt minus 10pt}
	\centering
	\caption{{\bf Results for Globe TSP.} Reported is the full tour accuracy.
	Standard deviations are over five restarts.}
	\label{tab:tsp-results}
	\begin{tabular}{@{\ }ccccc@{\ }}
	& \multicolumn{2}{c}{Embedding TSP Solver}&\multicolumn{2}{c}{ResNet18}\\
	$k$ & Train \% & Test \% & Train \%& Test \%\\
	\hline
	5  &  $99.8\std{0.0}$& $99.2\std{0.1}$ &$100.0\std{0.0}$ & $1.9\std{0.6}$\\
	10 &  $99.8\std{0.1}$& $98.7\std{0.2}$ & $99.0\std{0.1}$ & $0.0\std{0.0}$\\
	20 &  $99.1\std{0.1}$& $98.4\std{0.4}$ & $98.8\std{0.3}$ & $0.0\std{0.0}$\\
	40 &  $97.4\std{0.2}$& $96.7\std{0.4}$ & $96.9\std{0.3}$ & $0.0\std{0.0}$\\
	\end{tabular}
\end{wraptable}

\paragraph{Results.} This architecture not only learns to extract the correct
TSP tours but also learns the correct representations. Quantitative evidence is
presented in \tbref{tab:tsp-results}, where we see that the learned locations
generalize well and lead to correct TSP tours also on the test set and also on
somewhat large instances (note that there are $39! \approx 10^{46}$ admissible
TSP tours for $k=40$). The baseline architecture only memorizes the training
set. Additionally, we can extract the suggested locations of world capitals and
compare them with reality. To that end, we present \figref{fig:tsp:oceania},
where the learned locations of 10 capitals in Southeast Asia are displayed.

\subsection{MNIST Min-cost Perfect Matching}

\paragraph{Problem input and output.} The training dataset for problem {\sc
PM}$(k)$ consists of 10000 examples where the input to each example is a set of
$k^2$ digits drawn from the MNIST dataset arranged in a $k\times k$ grid. For
computing the label, we consider the underlying $k\times k$ grid graph (without
diagonal edges) and solve a {\sc min-cost-perfect-matching} problem, where edge
weights are given simply by reading the two vertex digits as a two-digit number
(we read downwards for vertical edges and from left to right for horizontal
edges).  The optimal perfect matching (\ie the label) is encoded by an
indicator vector for the subset of the selected edges, see example in
\figref{fig:pm}.

\paragraph{Architecture.} The grid image is the input of a convolutional neural
network which outputs a grid of vertex weights.  These weights are transformed
into edge weights as described above and given to the solver.  The loss
function is Hamming distance between solver output and the true label.

\begin{wraptable}{r}{.63\textwidth}
	\centering
	\caption{{\bf Results for MNIST Min-cost perfect matching.} Reported is the
	accuracy of predicting an optimal matching. Standard deviations are over
	five restarts.}
	\label{tab:pm-results}
	\begin{tabular}{@{\ }cc@{\ \ }cc@{\ \ }c@{\ }}
	& \multicolumn{2}{c}{Embedding Blossom V}&\multicolumn{2}{c}{ResNet18}\\
	$k$ & Train \% & Test \% & Train \%& Test \%\\
	\hline
	4  & $99.97\std{0.01}$ & $98.32\std{0.24}$ & $100.0\std{0.0}$ & $92.5\std{0.3}$\\
	8  & $99.95\std{0.04}$ & $99.92\std{0.01}$ & $100.0\std{0.0}$ & $ 8.3\std{0.8}$\\
	16 & $99.02\std{0.84}$ & $99.06\std{0.57}$ & $100.0\std{0.0}$ & $ 0.0\std{0.0}$\\
	24 & $95.63\std{5.49}$ & $92.06\std{7.97}$ & $96.1\std{0.5}$ & $ 0.0\std{0.0}$\\
	\end{tabular}
\end{wraptable}

\paragraph{Results.} The architecture containing the solver is capable of good
generalizations suggesting that the correct representation is learned. The
performance is good even on larger instances and despite the presence of noise
in supervision -- often there are many optimal matchings. In contrast, the
ResNet18 baseline only achieves reasonable performance for the simplest case
{\sc PM}$(4)$.  The results are summarized in \tbref{tab:pm-results}.

\setlength{\picHeight}{0.15\linewidth}
\begin{figure}[htb]
  \centering
  \begin{subfigure}[c]{0.5\linewidth}
    \centering
		\parbox[b][\picHeight][c]{1em}{\rotatebox{90}{Input}}
    \includegraphics[height=\picHeight]{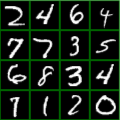}
		\vbox to\picHeight{\vfil\hbox{\LARGE$\to$}\vfil}
		\parbox[b][\picHeight][c]{0.7\picHeight}{
		\centering
		$$
			\begin{pmatrix}
				0\\ 1\\ \vdots\\ 1\\ 0
			\end{pmatrix}
		$$}
    \caption{}
    \label{fig:pm_dataset}
  \end{subfigure}
	\hspace{.09\linewidth}
  \begin{subfigure}[c]{0.3\linewidth}
    \centering
    \includegraphics[height=\picHeight]{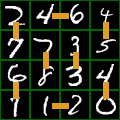}
    \caption{}
    \label{fig:pm_perfect_matching}
  \end{subfigure}
	\caption{Visualization of the {\sc PM} dataset. (\subref{fig:pm_dataset})
	shows the case of {\sc PM}$(4)$. Each input is a $4 \times 4$ grid of MNIST
	digits and the corresponding label is the indicator vector for the edges in
	the min-cost perfect matching. (\subref{fig:pm_perfect_matching}) shows the
	correct min-cost perfect matching output from the network. The cost of the
	matching is $348$ ($46+12$ horizontally and $27+45+40+67+78+33$ vertically).}
  \label{fig:pm}
\end{figure}

\section{Discussion}
We provide a unified mathematically sound algorithm to embed combinatorial
algorithms into neural networks. Its practical implementation is
straightforward and training succeeds with standard deep learning techniques.
The two main branches of future work are: 1) exploring the potential of newly
enabled architectures, 2) addressing standing real-world problems. The latter
case requires embedding approximate solvers (that are common in practice). This
breaks some of our theoretical guarantees but given their strong empirical
performance, the fusion might still work well in practice.

\section*{Acknowledgement}
We thank the International Max Planck Research School for Intelligent Systems (IMPRS-IS) for supporting Marin Vlastelica. We acknowledge the support from the German Federal Ministry of Education and Research (BMBF) through the Tübingen AI Center (FKZ: 01IS18039B). Additionally, we would like to thank Paul Swoboda and Alexander Kolesnikov for valuable feedback on an early version of the manuscript.

\bibliographystyle{iclr2020_conference}

\appendix

\section{Appendix}

\subsection{Guidelines for Setting the Values of $\lambda$.}\label{sec:app:guidlinelambda}

In practice, $\lambda$ has to be chosen appropriately, but we found its exact
choice uncritical (no precise tuning was required).  Nevertheless, note that
$\lambda$ should cause a noticeable disruption in the optimization problem from
equation \eqref{E:yl-def}, otherwise it is too likely that $y(w) =
y_\lambda(w)$ resulting in a zero gradient.  In other words, $\lambda$ should
roughly be of the magnitude that brings the two terms in the definition of $w'$
in Prop.~\ref{prop:linear} to the same order:
\begin{equation*}
  \lambda\approx\frac{\langle w \rangle}{\left \langle\pdiff{L}{y} \right\rangle}
\end{equation*}
where $\langle \cdot \rangle$ stands for the average.  This again justifies
that $\lambda$ is a {\bf true hyperparameter} and that there is no reason to
expect values around $\lambda \to 0^+$.

\subsection{Proofs}\label{sec:app:proofs}

\begin{proof}[Proof of Proposition \ref{prop:linear}]
Let us write $L=L(\hat y)$ and $\nabla L = \tfrac{\d L}{\d y}(\hat y)$, for
brevity.  Thanks to the linearity of $\cc$ and the definition of $f$, we have
\begin{equation*}
	\cc(\hat w,y) + \lambda f(y)
		= \hat w y + \lambda\bigl( L + \nabla L(y-\hat y)\bigr)
		= (\hat w + \lambda\nabla L)y + \lambda L - \lambda\nabla L\hat y
		= \cc(w',y) + \cc_0,
\end{equation*}
where
$\cc_0=\lambda L - \lambda\nabla L\hat y$
and $w'=\hat w + \lambda\nabla L$ as desired.
The conclusion about the points of minima then follows.
\end{proof}

Before we prove Theorem~\ref{T:f-lambda}, we make some preliminary
observations.  To start with, due to the definition of the solver, we have the
fundamental inequality
\begin{equation} \label{E:solver-min}
       \cc(w,y) \ge \cc\wy
              \quad\text{for every $w\in W$ and $y\in Y$}.
\end{equation}

\begin{observation} \label{O:c-cont}
The function $w\mapsto \cc\wy$ is continuous and piecewise linear.
\end{observation}

\begin{proof}
Since $\cc$'s are linear and distinct,
$\cc\wy$, as their pointwise minimum, has the desired properties.
\end{proof}

Analogous fundamental inequality
\begin{equation} \label{E:perturbed-solver-min}
       \cc(w,y) + \lambda f(y)
              \ge \cc\wyl + \lambda f\ylw
              \quad\text{for every $w\in W$ and $y\in Y$}
\end{equation}
follows from the definition of the solution
to the optimization problem \eqref{E:yl-def}.

A counterpart of Observation~\ref{O:c-cont} reads as follows.

\begin{observation} \label{O:clf-cont}
The function $w\mapsto \cc\wyl + \lambda f\ylw$ is continuous and piecewise
affine.
\end{observation}

\begin{proof}
The function under inspection is a pointwise minimum
of distinct affine functions $w\mapsto\cc(w,y) + \lambda f(y)$
as $y$ ranges $Y$.
\end{proof}

As a consequence of above-mentioned fundamental inequalities,
we obtain the following two-sided estimates on $f_\lambda$.

\begin{observation} \label{O:sendwich}
The following inequalities hold for $w \in W$
\begin{equation*}
       f\ylw\le f_\lambda(w) \le f\yw.
\end{equation*}
\end{observation}

\begin{proof}
Inequality \eqref{E:solver-min} implies that $\cc\wy-\cc\wyl \le 0$ and the
first inequality then follows simply from the definition of $f_\lambda$.  As
for the second one, it suffices to apply \eqref{E:perturbed-solver-min} to
$y=y(w)$.
\end{proof}

Now, let us introduce few notions that will be useful later in the proofs. For
a fixed $\lambda$, $W$ partitions into maximal connected sets $P$ on which
$y_\lambda(w)$ is constant (see \figref{fig:Wpartitioning}). We denote this
collection of sets by $\WW_\lambda$ and set $\WW=\WW_0$.

For $\lambda\in\R$ and $y_1\ne y_2\in Y$, we denote
\begin{equation*}
	F_\lambda(y_1,y_2)
		= \bigl\{w\in W: c(w,y_1) + \lambda f(y_1) = c (w,y_2) + \lambda f(y_2)
			\bigr\}.
\end{equation*}
We write $F(y_1,y_2)=F_0(y_1,y_2)$, for brevity. For technical reasons, we also
allow negative values of $\lambda$ here.

\begin{figure}[htb]
\centering
  \begin{subfigure}[t]{0.45\linewidth}
		\centering
		\includegraphics[width=\linewidth]{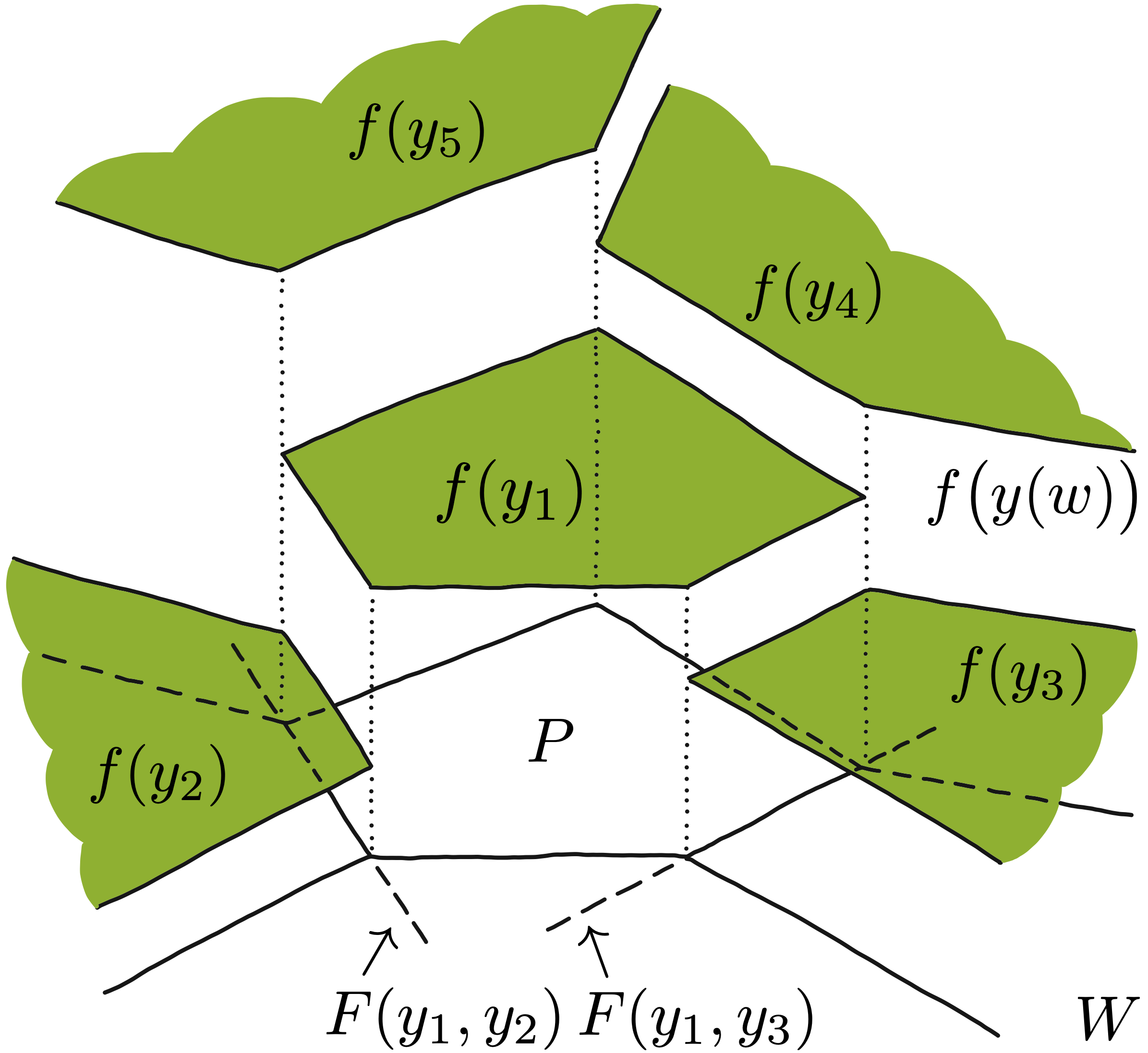}
		\caption{The situation for $\lambda=0$. We can see the polytope
		$P$ on which $y(w)$ attains $y_1\in Y$. The boundary of $P$
		is composed of segments of lines $F(y_1,y_k)$ for $k=2,\ldots,5$.}
		\label{fig:W0-partitioning}
	\end{subfigure}
	\hfil
  \begin{subfigure}[t]{0.45\linewidth}
		\centering
		\includegraphics[width=\linewidth]{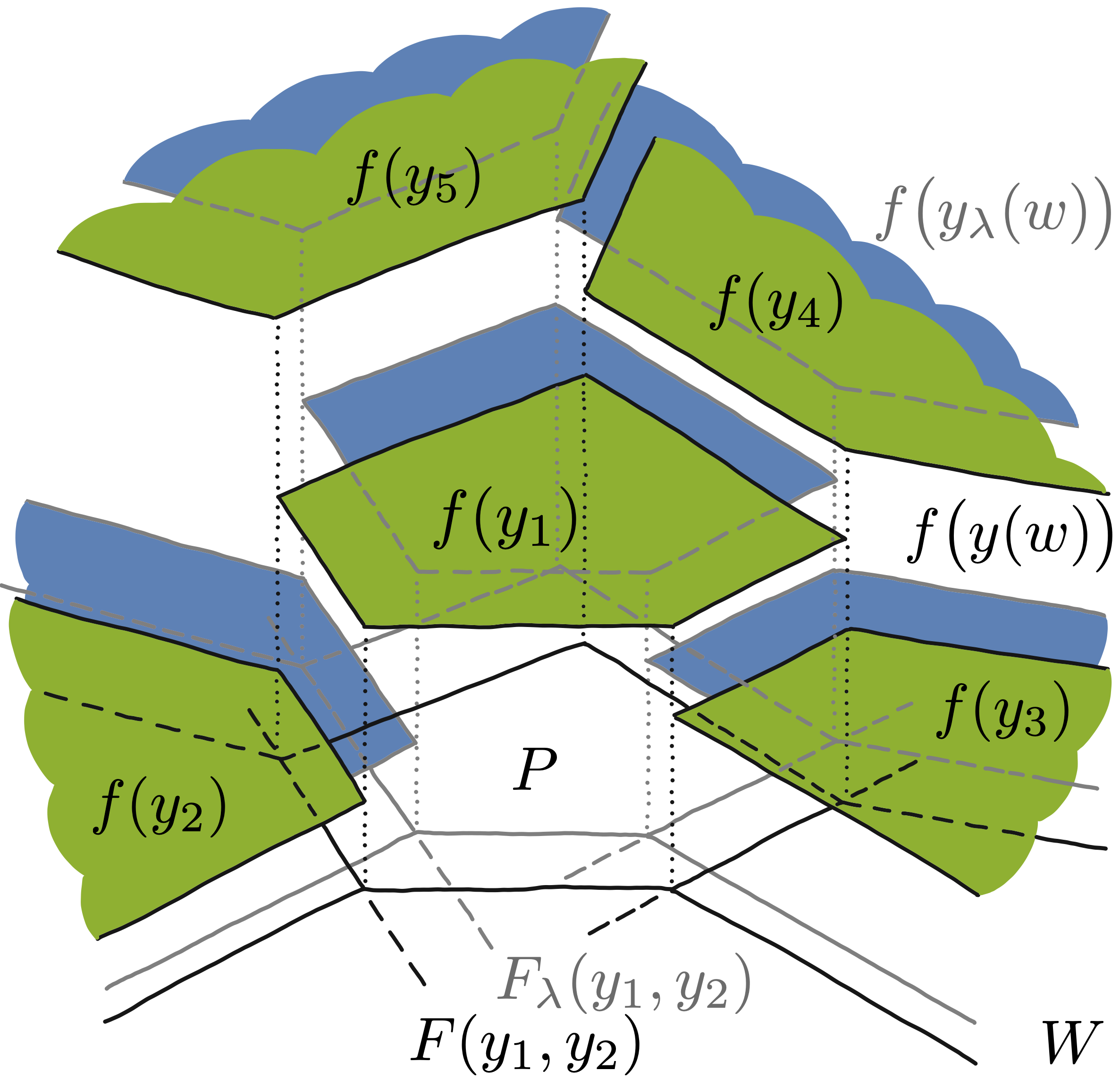}
		\caption{The same situation is captured for some relatively small
		$\lambda>0$.  Each line $F_\lambda(y_1,y_k)$ is parallel to its corresponding
		$F(y_1,y_k)$ and encompasses a convex polytope in $\WW_\lambda$.}
		\label{fig:Wlambda-partitioning}
	\end{subfigure}
\caption{ The family $\WW_\lambda$ of all maximal connected sets $P$ on which
$y_\lambda$ is constant.}
\label{fig:Wpartitioning}
\end{figure}

Note, that if $W=\R^N$, then $F_\lambda$ is a hyperplane since $\cc$'s are
linear. In general, $W$ may just be a proper subset of $\R^N$ and, in that case,
$F_\lambda$ is just the restriction of a hyperplane onto $W$.
Consequently, it may happen that $F_\lambda(y_1,y_2)$ will be empty
for some pair of $y_1$, $y_2$ and some $\lambda\in\R$.
To emphasize this fact, we say ``hyperplane in $W$''.
Analogous considerations should be taken into
account for all other linear objects. The note ``in $W$'' stands for
the intersection of these linear object with the set $W$.

\begin{observation} \label{O:facets}
Let $P\in\WW_\lambda$ and let $y_\lambda(w)=y$ for $w\in P$.
Then $P$ is a convex polytope in $W$, where the facets
consist of parts of finitely many hyperplanes
$F_\lambda(y,y_k)$ in $W$ for some $\{y_k\}\subset Y$.
\end{observation}

\begin{proof}
Assume that $W=\R^N$. The values of $y_\lambda$ may only change
on hyperplanes of the form $F_\lambda(y,y')$ for some $y'\in Y$.
Then $P$ is an intersection of corresponding half-spaces
and therefore $P$ is a convex polytope.
If $W$ is a proper subset of $\R^N$ the claim follows
by intersecting all the objects with $W$.
\end{proof}

\begin{observation} \label{O:distance}
Let $y_1, y_2\in Y$ be distinct.
If nonempty, the hyperplanes $F(y_1,y_2)$ and $F_\lambda(y_1,y_2)$ are
parallel and their distance is equal to $|\lambda| K(y_1,y_2)$, where
\begin{equation*}
	K(y_1,y_2)
		= \frac{|f(y_1) - f(y_2)|}{\|y_1 - y_2\|}.
\end{equation*}
\end{observation}

\begin{proof}
If we define a function $c(w)=\cc(w,y_1)-\cc(w,y_2)=w(y_1-y_2)$ and
a constant $C=f(y_2)-f(y_1)$, then
our objects rewrite to
\begin{equation*}
	F(y_1,y_2)=\{w\in W: c(w)=0\}
		\quad\text{and}\quad
	F_\lambda(y_1,y_2)=\{w\in W: c(w)=\lambda C\}.
\end{equation*}
Since $c$ is linear, these sets are parallel
and $F(y_1,y_2)$ intersects the origin.
Thus, the required distance is the distance of the hyperplane $F_\lambda(y_1,y_2)$
from the origin, which equals to $|\lambda C|/\|y_1-y_2\|$.
\end{proof}

As the set $Y$ is finite, there is a uniform upper bound $K$ on all values
of $K(y_1, y_2)$. Namely
\begin{equation} \label{E:bound-K}
	K = \max_{\substack{y_1,y_2\in Y\\y_1\ne y_2}} K(y_1,y_2).
\end{equation}

\subsubsection{Proof of Theorem~\ref{T:f-lambda}}

\begin{proof}[Proof of Property~A1]
Now, Property~A1 follows, since
\begin{equation*}
	f_\lambda(w)
		= \il \Bigl[ \cc\wyl + \lambda f\ylw \Bigr] - \il \cc\wy
\end{equation*}
and $f_\lambda$ is a difference of continuous and piecewise affine functions.
\end{proof}

\begin{proof}[Proof of Property~A2]
Let $0<\lambda_1\le \lambda_2$ be given.  We show that $\We[\lambda_2]\subseteq
\We[\lambda_1]$ which is the same as showing  $\Wd[\lambda_1]\subseteq
\Wd[\lambda_2]$. Assume that $w\in\We[\lambda_2]$, that is, by the definition
of $\We[\lambda_2]$ and $f_\lambda$,
\begin{equation} \label{E:Weq-l2}
	\cc\wy + \lambda_2 f\yw
		= \cc(w,y_2) + \lambda_2 f(y_2),
\end{equation}
in which we denoted $y_2=y_{\lambda_2}(w)$.
Our goal is to show that
\begin{equation} \label{E:Weq-l1}
	\cc\wy + \lambda_1 f\yw
		= \cc(w,y_1) + \lambda_1 f(y_1),
\end{equation}
where $y_1=y_{\lambda_1}(w)$ as this equality then guarantees that
$w\in\We[\lambda_1]$.  Observe that \eqref{E:perturbed-solver-min} applied to
$\lambda=\lambda_1$ and $y=y(w)$, yields the inequality ``$\ge$'' in
\eqref{E:Weq-l1}.

Let us show the reversed inequality.  By Observation~\ref{O:sendwich} applied
to $\lambda=\lambda_1$, we have
\begin{equation} \label{E:fyw-ge-fy1}
	f\yw \ge f(y_1).
\end{equation}
We now use \eqref{E:perturbed-solver-min} with $\lambda=\lambda_2$ and $y=y_1$,
followed by equality \eqref{E:Weq-l2} to obtain
\begin{align*}
	\cc(w,y_1) + \lambda_1 f(y_1)
		& = \cc(w,y_1) + \lambda_2 f(y_1)
			+ (\lambda_1-\lambda_2) f(y_1)
			\\
		& \ge \cc(w,y_2) + \lambda_2 f(y_2)
			+ (\lambda_1-\lambda_2) f(y_1)
			 \\
		& = \cc\wy + \lambda_2 f\yw
			+ (\lambda_1-\lambda_2) f(y_1)
			 \\
		& = \cc\wy + \lambda_1 f\yw
			+ (\lambda_2-\lambda_1)\bigl[f\yw - f(y_1)\bigr]
			 \\
		& \ge \cc\wy + \lambda_1 f\yw
\end{align*}
where the last inequality holds due to \eqref{E:fyw-ge-fy1}.

Next, we have to show that $\Wd\to\emptyset$ as $\lambda\to 0^+$, \ie that for
almost every $w\in W$, there is a $\lambda>0$ such that $w\notin\Wd$.  To this
end, let $w\in W$ be given.  We can assume that $y(w)$ is a unique solution of
solver \eqref{E:solver}, since two solutions, say $y_1$ and $y_2$, coincide
only on the hyperplane $F(y_1,y_2)$ in $W$, which is of measure zero.  Thus,
since $Y$ is finite, the constant
\begin{equation*}
	c = \min_{\substack{y\in Y\\y\ne y(w)}} \bigr\{\cc(w,y)-\cc\wy\bigr\}
\end{equation*}
is positive. Denote
\begin{equation} \label{E:def-const-d}
	d = \max_{y\in Y} \bigr\{f\yw-f(y)\bigr\}.
\end{equation}
If $d>0$, set $\lambda<c/d$. Then, for every $y\in Y$ such that $f\yw > f(y)$,
we have
\begin{equation*}
	\lambda < \frac{\cc(w,y)-\cc\wy}{f\yw-f(y)}
\end{equation*}
which rewrites
\begin{equation} \label{E:tmp1}
	\cc\wy + \lambda f\yw
		< \cc(w,y) + \lambda f(y).
\end{equation}
For the remaining $y$'s, \eqref{E:tmp1} holds trivially for every $\lambda>0$.
Therefore, $y(w)$ is a solution of the minimization problem~\eqref{E:yl-def},
whence $y_\lambda(w)=y(w)$. This shows that $w\in\We$ as we wished.  If $d=0$,
then $f\yw\le f(y)$ for every $y\in Y$ and \eqref{E:tmp1} follows again.
\end{proof}

\begin{proof}[Proof of Property~A3]
Let $y_1\ne y_2\in Y$ be given.  We show that on the component of the set
\begin{equation} \label{E:wcap}
	\{w\in W: \text{$y(w)=y_1$ and $y_\lambda(w)=y_2$}\}
\end{equation}
the function $f_\lambda$ agrees with a $\delta$-interpolator, where $\delta\le
C\lambda$ and $C>0$ is an absolute constant.  The claim follows as there are
only finitely many sets and their components of the form \eqref{E:wcap} in
$\Wd$.

Let us set
\begin{equation*}
	h(w) = \cc(w,y_1)-\cc(w,y_2)
		\quad\text{for $w\in W$}
\end{equation*}
and
\begin{equation*}
	g(w) = f(y_2) - \il h(w).
\end{equation*}
The condition on $\cc$ tells us that $h$ is a non-constant affine function.  It
follows by the definition of $F(y_1,y_2)$ and $F_\lambda(y_1,y_2)$ that
\begin{equation} \label{E:h-attains-0}
	h(w)=0
		\quad\text{if and only if}
		\quad w\in F(y_1,y_2)
\end{equation}
and
\begin{equation} \label{E:h-attains-lambdas}
	h(w)=\lambda\bigl( f(y_2)-f(y_1) \bigr)
		\quad\text{if and only if}
		\quad w\in F_\lambda(y_1,y_2).
\end{equation}
By Observation~\ref{O:distance}, the sets $F$ and $F_\lambda$ are
parallel hyperplanes. Denote by $G$ the nonempty intersection of their
corresponding half-spaces in $W$. We show that $g$ is a $\delta$-interpolator
of $f$ on $G$ between $y_1$ and $y_2$, with $\delta$ being linearly controlled
by $\lambda$.

We have already observed that $g$ is the affine function ranging
from $f(y_1)$ -- on the set $F_\lambda(y_1,y_2)$ -- to
$f(y_2)$ -- on the set $F(y_1,y_2)$. It remains to show that
$g$ attains both the values $f(y_1)$ and $f(y_2)$
at most $\delta$-far from the sets $P_1$ and $P_2$, respectively,
where $P_k\in\WW$ denotes a component of the set $\{w\in W: y(w)=y_k\}$,
$k=1,2$.

\begin{figure}[htb]
\centering
  \begin{subfigure}[t]{0.45\linewidth}
		\centering
		\includegraphics[width=0.85\linewidth]{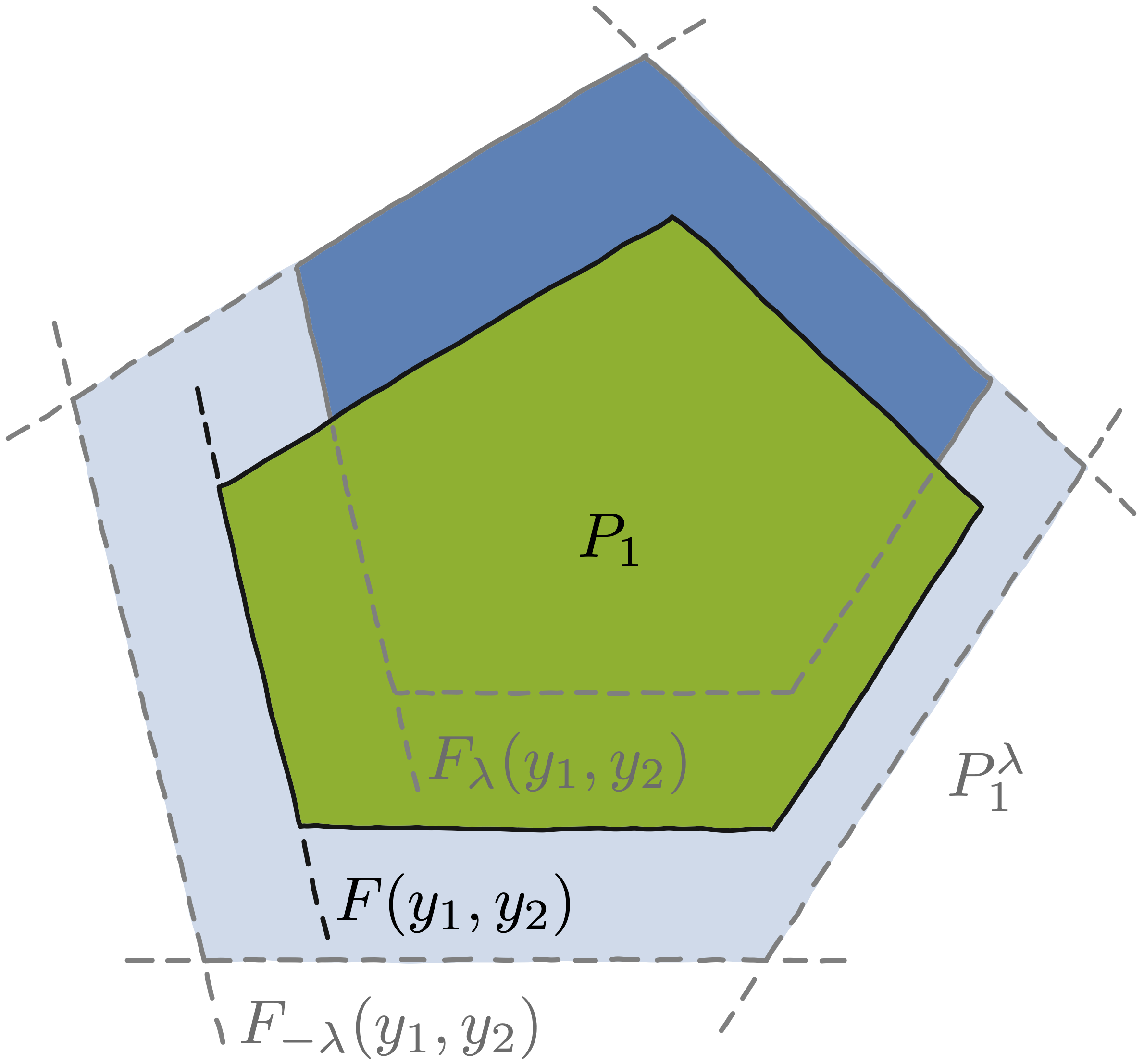}
		\caption{ The facets of $P_1$ consist of parts of hyperplanes $F(y_1,z_k)$
		in $W$.  Each facet $F(y_1,z_k)$ has its corresponding shifts $F_\lambda$
		and $F_{-\lambda}$, from which only one intersects $P$.  The polytope
		$P_1^\lambda$ is then bounded by those outer shifts.}
		\label{fig:polytope}
	\end{subfigure}
	\hfil
  \begin{subfigure}[t]{0.45\linewidth}
		\centering
		\includegraphics[width=\linewidth]{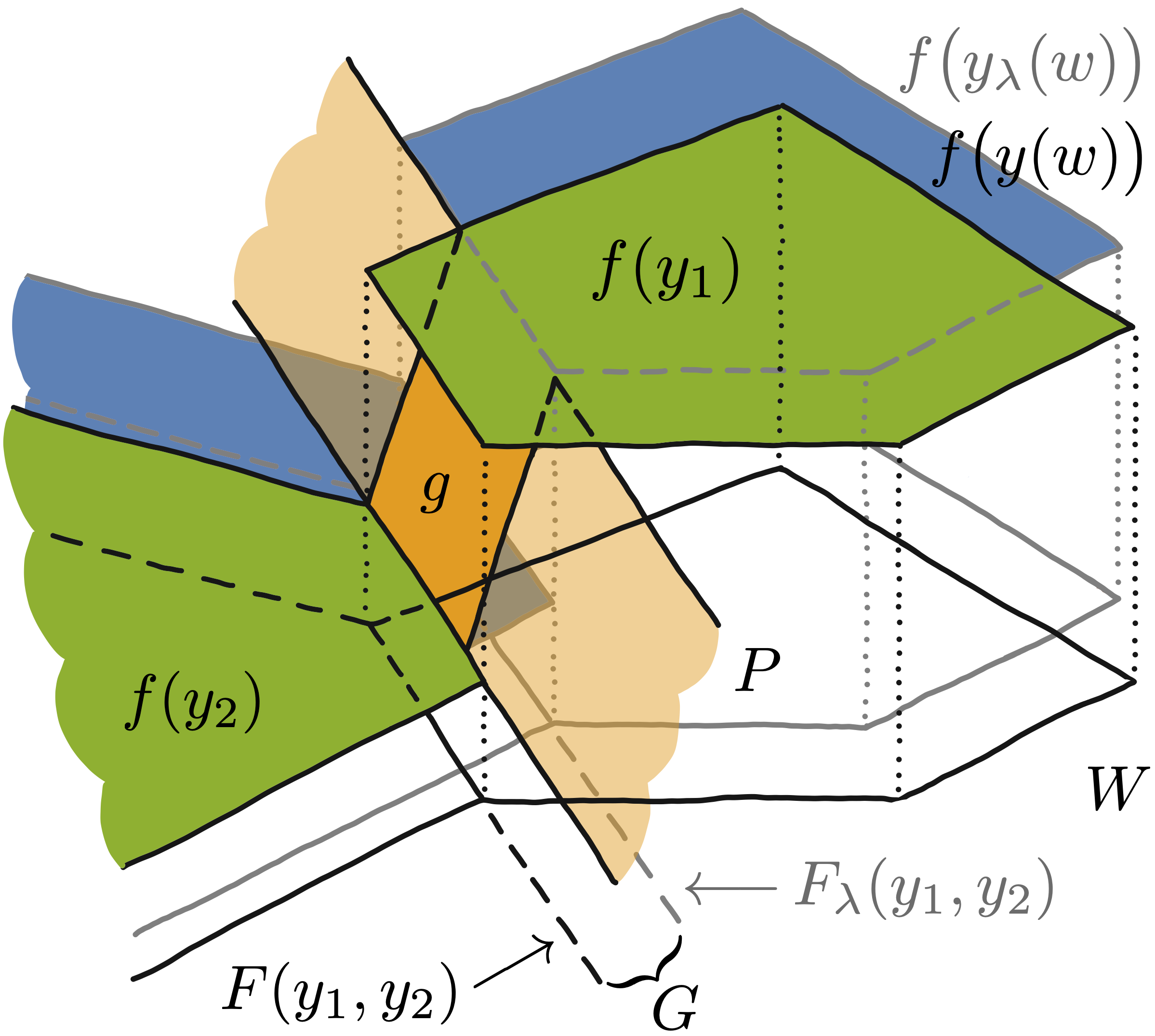}
		\caption{The interpolator $g$ attains the value $f(y_1)$ on a part of
		$F_\lambda(y_1,y_2)$ -- a border of the domain $G$.  The value $f(y_2)$ is
		attained on a part of $F(y_1,y_2)$ -- the second border of the strip $G$.}
		\label{fig:interpolator}
	\end{subfigure}
\caption{The polytopes $P_1$ and $P_1^\lambda$ and the interpolator $g$.}
\label{fig:A3proof}
\end{figure}

Consider $y_1$ first.  By Observation~\ref{O:facets}, there are
$z_1,\ldots,z_\ell\in Y$, such that facets of $P_1$ are parts of hyperplanes
$F(y_1,z_1),\ldots,F(y_1,z_\ell)$ in $W$.  Each of them separates $W$ into two
half-spaces, say $W_k^+$ and $W_k^-$, where $W_k^-$ is the half-space which
contains $P_1$ and $W_k^+$ is the other one. Let us denote
\begin{equation*}
	c_k(w) = \cc(w,y_1) - \cc(w,z_k)
		\quad\text{for $w\in W$ and $k=1,\dots,\ell$}.
\end{equation*}
Every $c_k$ is a non-zero linear function which is negative on $W_k^-$ and
positive on $W_k^+$.  By the definition of $y_1$, we have
\begin{equation*}
	\cc(w,y_1) + \lambda f(y_1)
		\le \cc(w,z_k) + \lambda f(z_k)
		\quad\text{for $w\in P_1$ and for $k=1,\ldots,\ell$},
\end{equation*}
that is
\begin{equation*}
	c_k(w) \le \lambda \bigl( f(z_k) - f(y_1) \bigr)
		\quad\text{for $w\in P_1$ and for $k=1,\ldots,\ell$}.
\end{equation*}
Now, denote
\begin{equation*}
	W_k^\lambda
		= \bigl\{w\in W: c_k(w)\le\lambda \bigl| f(z_k) - f(y_1) \bigr|\bigr\}
			\quad\text{for $k=1,\ldots,\ell$}.
\end{equation*}
Each $W_k^\lambda$ is a half-space in $W$ containing $W_k^-$ and hence $P_1$.
Let us set $P_1^\lambda = \bigcap_{k=1}^\ell W_k^\lambda$.  Clearly,
$P_1\subseteq P_1^\lambda$ (see \figref{fig:A3proof}).  By
Observation~\ref{O:distance}, the distance of the hyperplane $\bigl\{w\in W:
c_k(w)=\lambda \bigl| f(z_k) - f(y_1) \bigr|\bigr\}$ from $P_1$ is at most
$\lambda K$, where $K$ is given by \eqref{E:bound-K}.  Therefore, since all the
facets of $P_1^\lambda$ are at most $\lambda K$ far from $P_1$, there is a
constant $C$ such that each point of $P_1^\lambda$ is at most $C\lambda$ far
from $P_1$.

Finally, choose any $w_1\in P_1^\lambda\cap F_\lambda(y_1,y_2)$.  By
\eqref{E:h-attains-lambdas}, we have $g(w_1)=f(y_1)$, and by the definition of
$P_1^\lambda$, $w_1$ is no farther than $C\lambda$ away from $P_1$.

Now, let us treat $y_2$ and define the set $P_2^\lambda$ analogous to
$P_1^\lambda$, where each occurrence of $y_1$ is replaced by $y_2$.  Any
$w_2\in P_2^\lambda\cap F(y_1,y_2)$ has desired properties.  Indeed,
\eqref{E:h-attains-0} ensures that $g(w_2)=f(y_2)$ and $w_2$ is at most
$C\lambda$ far away from $P_2$.
\end{proof}

\subsection{Details of Experiments}\label{se:app:exp}

\subsubsection{Warcraft Shortest Path}

The maps for the dataset have been generated with a custom random generation
process by using 142 tiles from the Warcraft II tileset \citep{warcraft}. The
costs for the different terrain types range from $0.8$--$9.2$. Some example
maps of size $18\times 18$ are presented in \figref{fig:warcraft:examples}
together with a histogram of the shortest path lengths.  We used the first five
layers of ResNet18 followed by a max-pooling operation to extract the latent
costs for the vertices.

\begin{figure}[htb]
  \centering
  \begin{subfigure}[t]{0.68\linewidth}
		\centering
		\includegraphics[width=0.32\linewidth]{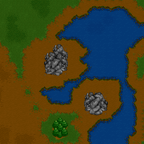}
		\includegraphics[width=0.32\linewidth]{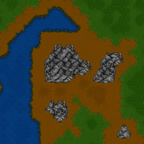}
		\includegraphics[width=0.32\linewidth]{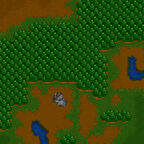}
		\caption{Three random example maps.}
		\label{fig:warcraft:examples}
	\end{subfigure}
	\hfil
  \begin{subfigure}[t]{0.30\linewidth}
		\centering
		\includegraphics[width=\linewidth]{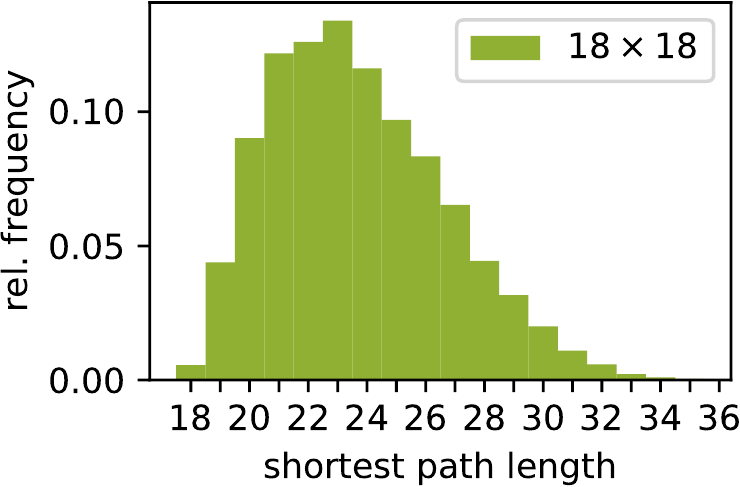}
		\caption{the shortest path distribution in the training set. All possible
		path lengths (18-35) occur.}
		\label{fig:warcraft:histogram}
	\end{subfigure}
  \caption{Warcraft {\sc SP}$(18)$ dataset.}
  \label{fig:warcraft}
\end{figure}

Optimization was carried out via Adam optimizer \citep{Adam} with scheduled
learning rate drops dividing the learning rate by $10$ at epochs $30$ and $40$.
Hyperparameters and model details are listed in \tbref{tab:warcraft:hyper}

\begin{table*}[htb]
  \centering
  \caption{Experimental setup for Warcraft Shortest Path.}
  \label{tab:warcraft:hyper}
	\begin{tabular}{@{\ }cccccc@{\ }}
    \textbf{k}
		& \textbf{Optimizer(LR)}
		& \textbf{Architecture} 
		& \textbf{Epochs}
		& \textbf{Batch Size}
		& $\mathbf{\lambda}$
    \\\hline
    12, 18, 24, 30 &  Adam($5 \times 10^{-4}$) & subset of ResNet18 & $50$ & $70$  & $20$
  \end{tabular}
\end{table*}

\subsubsection{MNIST Min-cost Perfect Matching}

The dataset consists of randomly generated grids of MNIST digits that are
sampled from a subset of 1000 digits of the full MNIST dataset.  We trained a
fully convolutional neural network with two convolutional layers followed by a
max-pooling operation that outputs a $k \times k$ grid of vertex costs for each
example.  The vertex costs are transformed into the edge costs via the known
cost function and the edge costs are then the inputs to the Blossom V solver
\citep{blossom-edmonds} as implemented in \citep{Kolmogorov2009}.

Regarding the optimization procedure, we employed the Adam optimizer along with
scheduled learning rate drops dividing the learning rate by $10$ at epochs $10$
and $20$, respectively. Other training details are in \tbref{app:tab:pm}. Lower
batch sizes were used to reduce GPU memory requirements.

\begin{table*}[ht]
  \centering
  \caption{Experimental setup for MNIST Min-cost Perfect Matching.}
  \label{app:tab:pm}
	\begin{tabular}{@{\ }cccccc@{\ }}
    \textbf{k}
		& \textbf{Optimizer(LR)}
		& \makecell{\textbf{Architecture}\\
								$[$channels, kernel size, stride$]$}
		& \textbf{Epochs}
		& \textbf{Batch Size}
		& $\mathbf{\lambda}$
		\\\hline
    4, 8&  Adam($10^{-3}$) & $[[20,5,1], [20,5,1]]$ & $30$ & $70$  & $10$\\
    16  &  Adam($10^{-3}$) & $[[50,5,1], [50,5,1]]$ & $30$ & $40$  & $10$ \\
    24  &  Adam($10^{-3}$) & $[[50,5,1], [50,5,1]]$ & $30$ & $30$  & $10$
  \end{tabular}
\end{table*}

\subsubsection{Globe Traveling Salesman Problem}

For the Globe Traveling Salesman Problem we used a convolutional neural network
architecture of three convolutional layers and two fully connected layers.  The
last layer outputs a vector of dimension $3k$ containing the $k$ 3-dimensional
representations of the respective countries' capital cities.  These
representations are projected onto the unit sphere and the matrix of pairwise
distances is fed to the TSP solver.

The high combinatorial complexity of TSP has negative effects on the loss
landscape and results in many local minima and high sensitivity to random
restarts. For reducing sensitivity to restarts, we set Adam parameters to
$\beta_1=0.5$ (as it is done for example in GAN training
\citep{radford2015unsupervised}) and $\epsilon=10^{-3}$.

The local minima correspond to solving planar TSP as opposed to spherical TSP.
For example, if all cities are positioned to almost identical locations, the
network can still make progress but it will never have the incentive to spread
the cities apart in order to reach the global minimum. To mitigate that, we
introduce a repellent force between epochs 15 and 30. In particular, we set
$$L_{\textrm{rep}} = \mathop{\mathbb{E}}_{i \neq j} e^{-\|x_i - x_j\|}$$ where
$x_i \in \R^3$ for $i = 1, \dots, k$ are the positions of the $k$ cities on the
unit sphere. The regularization constants $C_k$ were chosen as $2.0, 3.0, 6.0,$
and $20.0$ for $k \in \{5, 10, 20, 40\}$.

For fine-tuning we also introduce scheduled learning rate drops where we divide
the learning rate by $10$ at epochs $80$ and $90$.

\begin{table*}[ht]
  \centering
  \caption{Experimental setup for the Globe Traveling Salesman Problem.}
  \label{app:tab:tsp}
	\begin{tabular}{@{\ }cccccc@{\ }}
    \textbf{k}
		& \textbf{Optimizer(LR)}
		& \makecell{\textbf{Architecture}\\
								$[$channels, kernel size, stride$]$,\\
								linear layer size}
		& \textbf{Epochs}
		& \textbf{Batch Size}
		& $\mathbf{\lambda}$
		\\\hline
    5, 10, 20 &  Adam($10^{-4}$) & $[[20,4,2], [50,4,2], 500]$ & $100$ & $50$  & $20$  \\
    40 &  Adam($5 \times 10^{-5}$) & $[[20,4,2], [50,4,2], 500]$ & $100$ & $50$  & $20$
  \end{tabular}
\end{table*}

In \figref{fig:tsp:oceania}, we compare the true city locations with the ones
learned by the hybrid architecture. Due to symmetries of the sphere, the
architecture can embed the cities in any rotated or flipped fashion. We resolve
this by computing ``the most favorable'' isometric transformation of the
suggested locations. In particular, we solve the orthogonal Procrustes problem
\citep{oro2736}
\begin{equation*}
	R^* = \argmin_{R: R^TR = I} \|RX - Y\|^2
\end{equation*}
where $X$ are the suggested locations, $Y$ the true locations, and $R^*$ the
optimal transformation to apply. We report the resulting offsets in kilometers
in \tbref{app:tab:tsp_gps}.

\begin{table*}[ht]
  \centering
  \caption{Average errors of city placement on the Earth.}
  \label{app:tab:tsp_gps}
	\begin{tabular}{@{\ }ccccc@{\ }}
    k & 5 & 10 & 20 & 40 \\
    \hline
    Location offset (km)& $69 \pm 11$  & $19 \pm 5$ & $11 \pm 5$ & $58 \pm 7$
  \end{tabular}
\end{table*}

\subsection{Traveling Salesman with an Approximate Solver}

Since approximate solvers often appear in practice where the combinatorial
instances are too large to be solved exactly in reasonable time, we test our
method also in this setup.  In particular, we use the approximate solver
(OR-Tools~\citep{ortools}) for the Globe TSP.  We draw two conclusions from the
numbers presented below in \tbref{app:tab:tsp:approx}.
\begin{enumerate}
\item The choice of the solver matters. Even if OR-Tools is fed with the ground
truth representations (i.e. true locations) it does not achieve perfect results
on the test set (see the right column). We expect, that also in practical
applications, running a suboptimal solver (e.g. a differentiable relaxation)
substantially reduces the maximum attainable performance.
\item The suboptimality of the solver didn't harm the feature extraction -- the
point of our method. Indeed, the learned locations yield performance that is
close to the upper limit of what the solver allows (compare the middle and the
right column).
\end{enumerate}

\begin{table*}[ht]
  \centering
	\caption{{\bf Perfect path accuracy} for Globe TSP using the approximate solver
	OR-Tools~\citep{ortools}.  The maximal achievable performance is in the
	right column, where the solver uses the ground truth city locations.}
	\label{app:tab:tsp:approx}
	\begin{tabular}{@{\ }cccc@{\ }}
			& \multicolumn{2}{c}{Embedding OR-tools}
			& OR-tools on GT locations
			\\
			$k$ & Train \% & Test \% & Test \% \\
			\hline
			5  &  $99.8\std{0.0}$& $99.3\std{0.1}$ &$100.0$\\
			10 &  $84.3\std{0.2}$& $84.4\std{0.2}$ &$88.6$ \\
			20 &  $49.2\std{0.2}$& $48.6\std{0.8}$ &$54.4$ \\
			40 &  $14.6\std{0.1}$& $15.1\std{0.3}$ &$15.2$ \\ 
  \end{tabular}
\end{table*}

\end{document}